\documentclass[10pt,twocolumn,letterpaper]{article}

\usepackage{cvpr}
\usepackage{times}
\usepackage{epsfig}
\usepackage{graphicx}
\usepackage{amsmath}
\usepackage{amssymb}
\usepackage[thmmarks]{ntheorem}
\usepackage{cases}
{
\theoremstyle{nonumberplain}
\theoremheaderfont{\bfseries}
\theorembodyfont{\normalfont}
\theoremsymbol{\mbox{$\Box$}}
\newtheorem{proof}{Proof}

}
\newtheorem{thm}{Theorem}
\newtheorem{lemma}{Lemma}
\newtheorem{remark}{Remark}

\usepackage[breaklinks=true,bookmarks=false]{hyperref}

\cvprfinalcopy % *** Uncomment this line for the final submission

\def\cvprPaperID{****} % *** Enter the CVPR Paper ID here

\begin{document}

%%%%%%%%% TITLE
\title{Geometric Interpretation of side-sharing and point-sharing solutions in the P3P Problem}

\author{Bo Wang\\
Institute of Automation, Chinese Academy of Sciences \\
{\tt\small bo.wang@ia.ac.cn}
\and
Hao Hu \\
Ming Hsieh Department of Electrical Engineering, University of Southern California\\
{\tt\small huhao@usc.edu}
\and
Caixia Zhang \\
College of Science, North China University of Technology\\
{\tt\small zhangcx@ncut.edu.cn}
}

\maketitle
%\thispagestyle{empty}

%%%%%%%%% ABSTRACT
\begin{abstract}
It is well known that the P3P problem could have 1,  2,  3 and at most 4 positive solutions under different configurations among its 3 control points and the position of the optical center. Since in any real applications,  the knowledge on the exact number of possible solutions is a prerequisite for selecting the right one among all the possible solutions,  the study on the phenomenon of multiple solutions in the P3P problem has been an active topic . In this work,  we provide some new geometric interpretations on the multi-solution phenomenon in the P3P problem,  our main results include: (1): The necessary and sufficient condition for the P3P problem to have a pair of side-sharing solutions is the two optical centers of the solutions both lie on one of the 3 vertical planes to the base plane of control points; (2): The necessary and sufficient condition for the P3P problem to have a pair of point-sharing solutions is the two optical centers of the solutions both lie on one of the 3 so-called skewed danger cylinders;(3): If the P3P problem has other solutions in addition to a pair of side-sharing ( point-sharing) solutions,  these remaining solutions must be a point-sharing ( side-sharing ) pair. In a sense,  the side-sharing pair and the point-sharing pair are companion pairs. In sum,  our results provide some new insights into the nature of the multi-solution phenomenon in the P3P problem,  in addition to their academic value,  they could also be used as some theoretical guidance for practitioners in real applications to avoid occurrence of multiple solutions by properly arranging the control points.

\end{abstract}

%%%%%%%%% BODY TEXT
\section{Introduction}

The Perspective-3-Point Problem,  or P3P problem,  was first introduced by Grunert \cite{Grunert1841} in 1841,  and popularized in computer vision community a century later by mainly the Fishler and Bolles' work in 1981 \cite{RANSAC1981}. Since it is the least number of points to have a finite number of solutions,  it has been widely used in various fields (\cite{high-precision2006},\cite{Lowe1991},\cite{Nister2007},\cite{Quan1999},\cite{Tang2009}),  either for its minimal demand in restricted working environment,  such as robotics and aeronautics,  or for its computational efficiency of robust pose estimation under the random-sampling paradigm,  for example,  RANSAC-like algorithms.

The P3P problem can be defined as: Given the perspective projections of three control points with known coordinates in the world system and a calibrated camera,  determine the position and orientation of the camera in the world system. It is shown that \cite{Haralick1994} the P3P problem could have 1, 2, 3 or at most 4 solutions depending on the configuration between the optical center and its 3 control points. If the optical center and 3 control points happen to be cocyclic,  the problem becomes degenerated,  and an infinitely large number of solutions could be possible.

Since in any real applications,  a unique solution is always desirable,  and if the multiple solutions are unavoidable,  the knowledge on the exact number of possible solutions is a prerequisite for selecting the right solution from all the possible ones,  the study on the phenomenon of multiple solutions in the P3P problem has been an active topic since its very inception. In the literature,  the research on the multi-solution phenomenon in the P3P problem can be broadly divided into two directions. The one is so-called the algebraic approach,  the other is the geometric approach.

For the algebraic approach,  it basically boils down in the literature to the analysis of the roots distribution of the resulting quartic equation from the 3 original constraints. The related works include (\cite{Faugere2008},\cite{Gaoxiaoshan2003},\cite{Gaoxiaoshan2006},\cite{Rieck2012JMIV},\cite{Rieck2012VISAPP},\cite{Rieck2014},\cite{Wolfe1991},\cite{wuyihong2006}): For example,  Gao \cite{Gaoxiaoshan2003} provided the complete solution set and the probability \cite{Gaoxiaoshan2006} of having 1,  2,  3 and 4 solutions of the P3P problem; Rieck systematically analyzed solutions around the danger cylinder \cite{Rieck2012JMIV},\cite{Rieck2014},  and proposed a new algorithm for camera pose estimation close to the danger cylinder.

\indent For the geometric approach,  Zhang and Hu\cite{cylinder2006} showed that if the optical center lies on the danger cylinder,  the corresponding P3P problem must have 3 different solutions,  and that if the optical center lies on one of the three vertical planes to the base plane of control points,  the corresponding P3P problem must have 4 different solutions \cite{Wolfe2008},\cite{four2005}. Sun \cite{sunfengmei2010} found that when the optical center lies on the two intersecting vertical lines between the danger cylinder and one of the three vertical planes,  either the side-sharing pair or the point-sharing pair of solutions degenerates into a double solution. However such geometric conclusions do not indicate whether besides the danger cylinder or the 3 vertical planes,  the P3P problem could still have 3 or 4 different solutions if the optical center lies on other places.\\
\indent
In this work,  we push the frontier a step further for the geometrical approach,  we not only find the geometrically meaningful sufficient conditions but also necessary conditions for the multiple solutions of the P3P problem. To our knowledge,  we are the first to obtain both the necessary and sufficient conditions with clear geometrical interpretations in the literature. \\
\indent
This work is organized as follows: The next section introduces some preliminaries on the P3P problem. The main results are listed in Section 3: In subsection 3.1, three theorems are related to algebraic and geometric conditions of a pair of side-sharing solutions;In subsection 3.2, three theorems are related to algebraic and geometric conditions of a pair of point-sharing solutions;
In subsection 3.3, we propose the relation between the point-sharing and side-sharing solutions, that is: the side-sharing pair and the point-sharing pair are often companion pairs,  and followed by some conclusions in Section 4.

\section{Preliminaries}
For the notational convenience and a better understanding of our main results in the next section,  here we first list some preliminaries on the P3P problem.
\subsection{The three basic constraints of the P3P problem and its transformed two conics}

\begin{figure}[t]
\begin{center}
\includegraphics[width=0.5\linewidth]{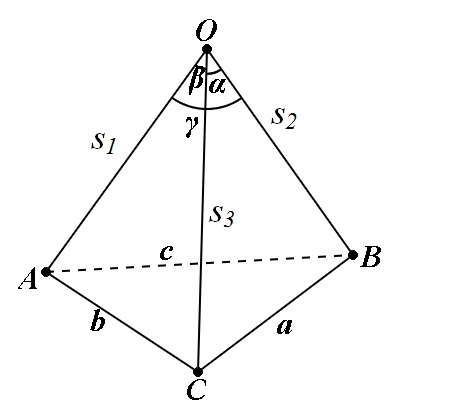}
\end{center}
   \caption{P3P problem definition: $A, B, C$ are the 3 control points,  $O$ is the optical center.}
\label{Fig1}
\end{figure}

As shown in Fig.\ref{Fig1},  $A, B, C$ are the three control points with known distance $a=|BC|, b=|AC|, c=|AB|$,  $O$ is the optical center.  Since the camera (under the pinhole model) is calibrated,  the three subtended angles $\alpha, \beta, \gamma$ of the projection rays can be computed,  hence can be considered as known entities,  then by the Law of Cosines,  we immediately have the following 3 basic constraints on the three unknown $s_1=|OA|, s_2=|OB|, s_3=|OC|$ in (\ref{equ1}).
\begin{gather}\label{equ1}
  \begin{cases}
    s_1^2 +s_2^2 -2cos\gamma s_1s_2 =c^2 \\
    s_1^2 +s_3^2 -2cos \beta s_1s_3 =b^2  \\
    s_2^2 +s_3^2 -2cos\alpha s_2s_3 =a^2
  \end{cases}
\end{gather}
\indent Hence the P3P problem refers to determine such positive triplets $(s_1, s_2, s_3)$ satisfying the 3 basic constraints in (\ref{equ1}). And the so-called multiple-solution phenomenon in the P3P problem means for a given P3P problem,  more than one positive triplets exist which all satisfy the 3 basic constraints in (\ref{equ1}). Clearly,  given a positive triplet (or a solution),  its optical center location with respect to the 3 control points can be uniquely determined.

Let $u=s_2/s_1$,  and $v=s_3/s_1$,  then the three constraints in (\ref{equ1}) can be transformed into the following two conics in (\ref{equ2}) with $u$ and $v$ as variables:
{\small
\begin{gather}\label{equ2}
  \begin{cases}
    C_1:(a^2-b^2)v^2+2b^2cos\alpha uv-b^2u^2-2a^2cos\beta v+a^2=0\\
    C_2:(a^2-c^2)u^2+2c^2cos\alpha uv-c^2v^2-2a^2cos\gamma u+a^2=0
  \end{cases}
\end{gather}
}
Clearly the above two conics can intersect at most 4 points with positive coordinates,  and since each positive intersecting point can give rise to a positive solution of the P3P problem,  the P3P problem can have at most 4 different positive solutions. For the presentational purpose,  hereinafter,  a solution of the P3P problem means a positive solution. As shown later in the text,  these two conics are the characteristic conics of the P3P problem,  and the configurations of their intersecting points directly determine the multi-solution phenomenon in the P3P problem.

\subsection{The danger cylinder and the 3 vertical planes}

\begin{figure}[t]
\begin{center}
\includegraphics[width=0.4\linewidth]{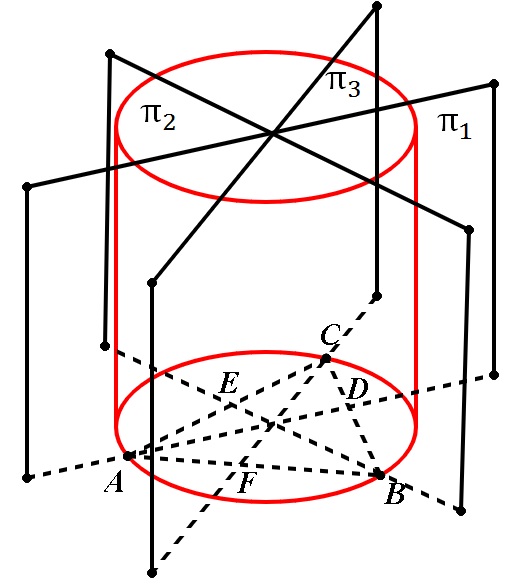}
\end{center}
   \caption{The danger cylinder and three vertical planes.}
\label{Fig3first}
\end{figure}

As shown in Fig.\ref{Fig3first},  the vertical cylinder passing through the 3 control points $A, B, C$ is called the danger cylinder. It is shown that if the optical center lies on the danger cylinder,  the solution of the P3P problem is unstable \cite{cylinder1966}. Zhang \cite{cylinder2006} showed that if the optical center lies on the danger cylinder,  there are always three different solutions of the P3P problem. Later on,  Sun \cite{sunfengmei2010} showed that the necessary and sufficient condition for the optical center lies on the danger cylinder is the nullness of the determinant of the Jacobian of the 3 constraints in (\ref{equ1}). In \cite{Rieck2014} it is shown that the optical center lying on the danger cylinder is the necessary and sufficient condition for the P3P problem to have repeated solutions.

As shown in Fig.\ref{Fig3first},  $AD,  BE,  CF$ are the three altitudes of the triangle $ABC$,  the three vertical planes to the base plane of the triangle $ABC$ are the vertical planes going through $AD$,  or $BE$,  or $CF$,  denoted as $\pi_1$ ,  $\pi_2$ and $\pi_3$. In Section 3,  we show that when the optical center lies on one of these 3 vertical planes,  the corresponding P3P problem must have a pair of side-sharing solutions,  and vice versa.
\subsection{Side-sharing solutions and point-sharing solutions}

\begin{figure}[t]
\begin{center}
\includegraphics[width=0.5\linewidth]{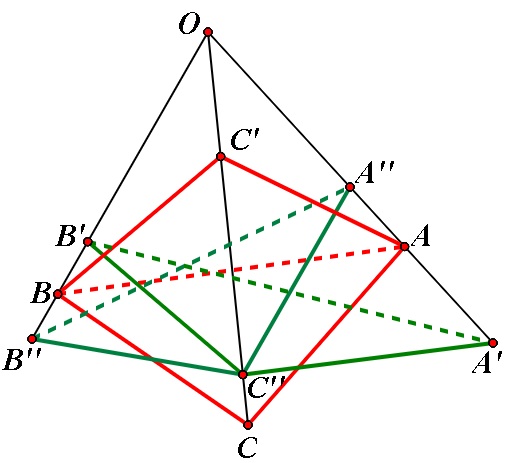}
\end{center}
   \caption{Side-sharing solutions and point-sharing solutions: $\{O, (A, B, C)\}$ and $\{O, (A, B, C')\}$ are a side-sharing pair with $AB$ being the shared side,  $\{O, (A', B', C'')\}$ and $\{O, (A'', B'', C'')\}$ are a point-sharing pair with $C''$ being the shared point.}
\label{Fig3first2}
\end{figure}

Geometrically,  a solution of P3P problem can be completely defined by the 3 control points $A, B, C$ and the optical center position $O$,  denoted as $\{O, (A, B, C) \}$. Given a P3P problem,  its two different solutions can be expressed either by superposing the 3 control points but with different optical center positions,  such as $\{O, (A, B, C) \}, \{O', (A, B, C)\}$,  or equivalently by superposing the optical centers $O$ and $O'$,  but with different positions of the control points at their respective projection rays,  denoted as $\{O, (A, B, C) \}, \{O, (A', B', C' ) \}$. For convenience, we call them as three control points unchanged expression and the optical center unchanged expression. In the following theorems, these two different expressions are interchangeably used for the easy representation purpose depending on the circumstance.

In this work,  a pair of side-sharing solutions and a pair of point-sharing solutions are frequently used. As shown in \cite{four2005},  a pair of side-sharing solutions refers to two such solutions that if their optical centers are superposed,  they must share two control points or share a common side of the control point triangle. In contrast,  a pair of point-sharing solutions refers to two such solutions that if their optical centers are superposed,  they must share a single common control point. As shown in Fig.\ref{Fig3first2},  $\{O, (A, B, C)\}$ and $\{O, (A, B, C')\}$ are a pair of side-sharing solutions with shared side $AB$,  and $\{O, (A', B', C'')\}$ and $\{O, (A'', B'', C'')\}$ are a pair of point-sharing solutions with shared point $C''$.
%-------------------------------------------------------------------------
\section{Main Results}
In this section, we show that the necessary and sufficient condition for the P3P problem to have a pair of side-sharing solutions is its optical center lies on one of the three vertical planes, and the necessary and sufficient condition for the P3P problem to have a pair of point-sharing solutions is its optical center lies on one of the three so-called skewed danger cylinders. In addition, we show that given a side-sharing pair (or a point-sharing pair) of solutions, if the P3P problem still has other solutions, these remaining solutions must be a pair of point-sharing (or side-sharing) solutions.

Before discussing our main results, we would at first needs the following statements:
\begin{thm}\label{thmu0v0}
If the two conics have an intersection  $(u_0, v_0)$ in the quadrant I, then $(u_0, v_0)$ is corresponding to one solution of the P3P problem.
\end{thm}

\begin{proof}
Since $(u_0, v_0)$ lies in the quadrant I, we have $u_0>0, v_0>0 $ and also have：
{\small
\begin{gather}
  %\begin{cases}
    (a^2-b^2)v_0^2+2b^2cos\alpha u_0v_0-b^2u_0^2-2a^2cos\beta v_0+a^2=0   \label{conic1} \\
    (a^2-c^2)u_0^2+2c^2cos\alpha u_0v_0-c^2v_0^2-2a^2cos\gamma u_0+a^2=0  \label{conic2}
  %\end{cases}
\end{gather}
}
It is obvious that equation (\ref{conic1}) and equation (\ref{conic2}) are equivalent to the following equations (\ref{conic1e})  and (\ref{conic2e})  respectively.

\begin{numcases}{}
    \frac{a^2}{u_0^2+v_0^2-2cos\alpha u_0v_0}=\frac {b^2}{1+v_0^2-2cos\beta v_0}    \label{conic1e} \\
    \frac{a^2}{u_0^2+v_0^2-2cos\alpha u_0v_0}=\frac {c^2}{1+u_0^2-2cos\gamma u_0}   \label{conic2e}
\end{numcases}

Considering the inequality $u_0^2+v_0^2-2cos\alpha u_0v_0>0 $, let
\begin{gather}
  %\begin{cases}
    s_{01}=\sqrt{\frac{a^2}{u_0^2+v_0^2-2cos\alpha u_0v_0}}    \label{s01}
  %\end{cases}
\end{gather}
where in the above double subscript, the first "0" and the second "1" indicate a specific solution and the component of the solution respectively, it is similar in the following notations. 
Let
\begin{gather}
  %\begin{cases}
    s_{02}=u_0s_{01}    \label{s02}\\ 
    s_{03}=v_0s_{01}    \label{s03}
  %\end{cases}
\end{gather}
%Then
%\begin{gather*}\label{equ1}
%  \begin{cases}
%    s_{01}^2 +s_{02}^2 -2cos\gamma s_{01}s_{02} =c^2 \\
%    s_{01}^2 +s_{30}^2 -2cos\beta s_{01}s_{30}  =b^2  \\
%    s_{02}^2 +s_{30}^2 -2cos\alpha s_{02}s_{30} =a^2
%  \end{cases}
%\end{gather*}
It is obvious that $s_{01}>0,s_{02}>0,s_{03}>0$ and $(s_{01},s_{02},s_{03})$ satisfy the 3 basic constraints in (\ref{equ1}). So $(s_{01},s_{02},s_{03})$ is one solution of the P3P problem.
\end{proof}

\begin{lemma} \label{lemmaRieck}
The necessary and sufficient condition for the P3P problem to have the repeated solutions is that the optical center lies in the danger cylinder.
\end{lemma}
The proof of this lemma refers to \cite{Rieck2012JMIV}.

\begin{lemma}
The necessary and sufficient condition for the P3P problem to have the repeated solutions is that the two conics derived from the P3P problem are tangential in the quadrant I.
\end{lemma}

\begin{proof}
 {\bf (Sufficiency)}
  If the two conics derived from the P3P problem are tangential in the quadrant I, then they have a couple of repeated positive solutions which are $(u_t, v_t)$ and $(u_t^{'}, v_t^{'})$. By Theorem \ref{thmu0v0}, The P3P problem has two solutions $(s_{t1},s_{t2},s_{t3})$  and  $(s_{t1}^{'},s_{t2}^{'},s_{t3}^{'})$ corresponding to $(u_t, v_t)$ and $(u_t^{'}, v_t^{'})$ respectively. By Equations (\ref{s10})--(\ref{s30}), $(s_{t1},s_{t2},s_{t3})= (s_{t1}^{'},s_{t2}^{'},s_{t3}^{'}) $,  so the P3P problem has the repeated solutions.\\
 {\bf (Necessity)}
 If the P3P problem has the repeated solutions which are $(s_{t1},s_{t2},s_{t3})$  and $(s_{t1}^{'},s_{t2}^{'},s_{t3}^{'})$  , then the corresponding $(u_t, v_t)$  and $(u_t^{'}, v_t^{'})$  are equal and locate in the quadrant I by Equations (\ref{s02}) and (\ref{s03}). This shows that the two conics are tangential in the quadrant I.
\end{proof}

\subsection{A side-sharing pair and the geometric meaning}
In this section, we show that three theorems about the side-sharing pair: Theorem \ref{thmsidetwo} is the necessary and sufficient condition,  Theorem \ref{thmsideone} is a sufficient condition,  Theorem \ref{thmsidegeometric} is the geometric meaning that the optical centers of side-sharing solutions must lie on one of three vertical planes.
\begin{figure}[t]
\begin{center}
\includegraphics[width=0.5\linewidth]{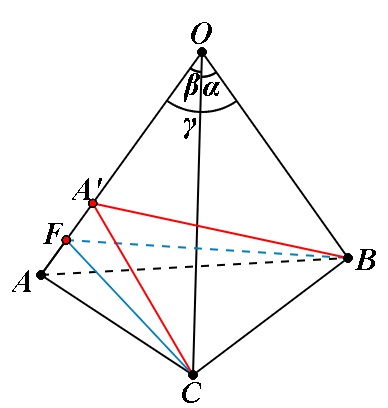}
\end{center}
   \caption{$\{O, (A, B, C)\}$ and $\{O, (A', B, C)\}$ are a side-sharing pair of solutions.}
\label{fig5first}
\end{figure}

\begin{thm}\label{thmsidetwo}
  Under the optical center unchanged expression, the necessary and sufficient condition for the P3P problem to have a side-sharing pair with a shared side $BC$ is both the corresponding $(u_1, v_1)$ and $(u_2, v_2)$  satisfy the equation $cos\gamma u-cos\beta v=0$.
\end{thm}

\begin{proof}
     {\bf (Necessity)}\\
   Let $(O,ABC)$  and $(O,A'BC)$  be a side-sharing pair with a shared side $BC$ (see Figure \ref{fig5first}) .\\
Then $|A'B|=|AB|$,$|A'C|=|AC|$.   Let $F$  be the midpoint of the segment $AA'$, then $CF \perp OA, BF \perp OA $, Clearly, $BC \perp OA$ and
\begin{gather*}
  \begin{cases}
    |OF|=|OB|cos\gamma \\
    |OF|=|OC|cos\beta
  \end{cases}
\end{gather*}
Let $s_{11}=|OA|, s_{2}=|OB|,s_{3}=|OC|,s_{21}=|OA'|$ be the side lengths.
We get
 \begin{gather}\label{side}
 % \begin{cases}
    s_{2}cos\gamma=s_{3}cos\beta
 % \end{cases}
\end{gather}
Divide each side of Equation (\ref{side}) by $s_{11}$ , $cos\gamma u_1-cos\beta v_1=0$ .\\
Divide each side of Equation (\ref{side}) by $s_{21}$ , $cos\gamma u_2-cos\beta v_2=0$ .

{\bf (Sufficiency)}
Since  both  $(u_1, v_1)$ and $(u_2, v_2)$  satisfy the equation $cos\gamma u-cos\beta v=0$, we have
\begin{gather*}
  \begin{cases}
    cos\gamma u_1-cos\beta v_1=0 \\
    cos\gamma u_2-cos\beta v_2=0
  \end{cases}
\end{gather*}
By definition, $u_i=\frac {s_{i2}}{s_{i1}},v_i=\frac {s_{i3}}{s_{i1}},i=1,2$.
We get
\begin{gather*}
  \begin{cases}
    cos\gamma s_{12}-cos\beta s_{13}=0 \\
    cos\gamma s_{22}-cos\beta s_{23}=0
  \end{cases}
\end{gather*}
So we have
\begin{gather*}
     \frac{s_{12}}{s_{22}}=\frac{s_{13}}{s_{23}}=k>0
\end{gather*}
Since
\begin{gather}\label{a}
 s_{12}^2+s_{13}^2-2cos\alpha s_{12}s_{13}=a^2
\end{gather}
Substitute $s_{12}=ks_{22}, s_{13}=ks_{23}$  into Equation (\ref{a}), we have
\begin{gather*}
 k^2(s_{22}^2+s_{23}^2-2cos\alpha s_{22}s_{23})=a^2
\end{gather*}
Since
\begin{gather}
s_{22}^2+s_{23}^2-2cos\alpha s_{22}s_{23}=a^2
\end{gather}
We get $k=1$.
So $s_{12}=s_{22},s_{13}=s_{23}$, that is, two solutions of the P3P problem corresponding to $(u_1, v_1)$ and $(u_2, v_2)$ are a side-sharing pair with a shared side $BC$ .
\end{proof}

\begin{thm}\label{thmsideone}
  Under the optical center unchanged expression, if $(u_1, v_1)$  satisfy $cos\gamma u-cos\beta v=0$ which correspond to a solution $(O,ABC)$  of the P3P Problem, and $\alpha < \angle BAC$, then the P3P problem has another solution which has a shared side $BC$  with $(O,ABC)$ , and the corresponding $(u_2, v_2)$  also satisfy $cos\gamma u-cos\beta v=0$.
\end{thm}
\begin{proof}
  {\bf (By construction)}\\
Let $s_{11}=|OA|,s_{12}=|OB|,s_{13}=|OC|$  be the side lengths corresponding to the solution $(O,ABC)$ .
Since $cos\gamma u_1-cos\beta v_1=0$,
We get  $cos\gamma s_{12}-cos\beta s_{13}=0$.
Furthermore, $2cos\gamma s_{12}-s_{11}=2cos\beta s_{13}-s_{11}$.\\
Let $s_{21}=2cos\gamma s_{12}-s_{11}$  or  $s_{21}=2cos\beta s_{13}-s_{11}$,
then
\begin{gather*}
\begin{split}
&s_{21}^2 +s_{12}^2 -2cos\gamma s_{21}s_{12}\\
&= (2cos\gamma s_{12}-s_{11})^2 +s_{12}^2 -2cos\gamma (2cos\gamma s_{12}-s_{11})s_{12}\\
&= s_{11}^2 +s_{12}^2 -2cos\gamma s_{11}s_{12}=c^2
\end{split}
\end{gather*}
Similarly,
\begin{gather*}
\begin{split}
&s_{21}^2 +s_{13}^2 -2cos\beta s_{21}s_{13}\\
&=(2cos\beta s_{13}-s_{11})^2 +s_{13}^2 -2cos\beta (2cos\beta s_{13}-s_{11})s_{13}\\
&=s_{11}^2 +s_{13}^2 -2cos\beta s_{11}s_{13}=b^2
\end{split}
\end{gather*}

So $(s_{21},s_{12},s_{13})$satisfy the three constraints of the P3P Problem. Thus if $s_{21}>0$, then $(s_{21},s_{12},s_{13})$ is another solution of the P3P problem which has a shared side $BC$  with $(s_{11},s_{12},s_{13})$ . Thus by Theorem \ref{thmsidetwo}, we will conclude that the corresponding  $(u_2, v_2)$  also satisfy $cos\gamma u-cos\beta v=0$.
\begin{figure}[t]
\begin{center}
\includegraphics[width=0.5\linewidth]{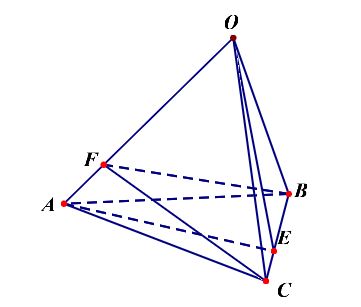}
\end{center}
   \caption{To construct another solution which has a shared side $BC$ with $(O,ABC)$.}
\label{figside}
\end{figure}
Let us next show that when $\alpha < \angle BAC$ , $s_{21}>0$.
As shown in Figure \ref{figside}, Let $CF\bot AO$ , $F$ be the foot of $C$ on segment $OA$ , then by
$s_{21}=2cos\beta s_{13}-s_{11}$, we get $s_{21}=2|OF|-|OA|=2|OF|-(|OF|+|FA|)=|OF|-|FA|$. Thus $s_{21}>0$  iff $|OF|>|FA|$. Considering $CF\bot AO$, $|OF|>|FA|$ iff $|OC|>|AC|$ .
Link $B$ and $F$, and by $cos\gamma s_{12}-cos\beta s_{13}=0$, we have $BF\bot AO$, then  $AO\bot BCF$, furthermore, $AO\bot BC$. Let $AE\bot BC$ and link $O$ and $E$, then $BC\bot AEO$, so $OE\bot BC$, then $|OC|=\sqrt{|OE|^2+|CE|^2}$,$|AC|=\sqrt{|AE|^2+|CE|^2}$, Thus $|OC|>|AC|$ iff $|OE|>|AE|$ .
By $\angle BAC>\alpha$, we have $\angle ABC +\angle ACB < \angle OBC +\angle OCB$, thus at least one holds of inequalities $\angle ABC  < \angle OBC <\pi/2 $  and $\angle ACB < \angle OCB <\pi/2$ (corresponding respectively to three situations, the first is when $E$  lies on the base side $BC$, the second is when $E$ lies on the extended side $BC$, the third is when $E$  lies on the extended side $CB$ ). Without loss of generality, let $\angle ABC  < \angle OBC <\pi/2 $ , since$|AE|=|BE|\tan\angle ABC $,$|OE|=|BE|\tan\angle OBC $, we know $|OE|>|AE|$.
\end{proof}
\begin{remark}
\item 1.The difference between Theorems \ref{thmsidetwo} and \ref{thmsideone} is that Theorem  \ref{thmsidetwo} gives the necessary and sufficient condition about a side-sharing pair, while Theorem \ref{thmsideone} gives how to construct another solution which has a common side with the known specific solution. Moreover, a side-sharing pair with a shared side $AB$ or $AC$ has similar result.

 \item 2.$\alpha < \angle BAC$	 means that the optical center $O$  lies outside the toroid by rotating the arc $BAC$ of circumcircle of the triangle $\triangle ABC$ around $BC$ .
\end{remark}
\begin{thm}\label{thmsidegeometric}
{\bf {(The geometric meaning of the side-sharing pair)}}\\
Under the optical center unchanged expression, if $(O,ABC)$  and  $(O,A'BC)$   are a side-sharing pair with a shared side $BC$ of the P3P problem, then both the optical centers $O$  and $O'$ of $(O,ABC)$  and  $(O',ABC)$ ,under three control points unchanged expression, locate on the vertical plane $\pi_1$, where $(O',ABC)$ is equivalent to $(O,A'BC)$.
\end{thm}
\begin{figure}[t]
\begin{center}
\includegraphics[width=0.5\linewidth]{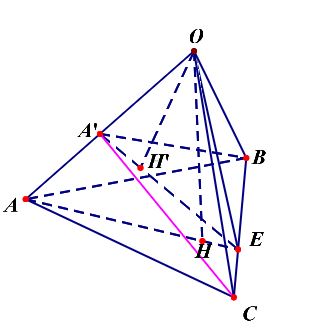}
\end{center}
   \caption{The geometric meaning of the side-sharing pair.}
\label{sidegeo}
\end{figure}
\begin{proof}
By the necessity of Theorem \ref{thmsidetwo}, we know $BC \bot OA$ .

As shown in Figure \ref{sidegeo}, Let $OH \bot \triangle ABC$, $H$ is the foot. Link $A$  and $H$, let the intersection of line $AH$  and line $BC$ is $E$, then by $BC \bot OH$ and $BC \bot OA$, $BC \bot AE$ is obvious. So the optical center $O$  of $(O,ABC)$   locate on the vertical plane $\pi_{BC}$.
Link $A'$ and $E$, by  $\triangle ABC$ is congruent to $\triangle A'BC$, $BC \bot A'E$  holds.
Let $OH \bot A'E$, $H'$ is the foot.  By  $BC \bot OA'$   and $BC \bot A'E$, we have $ OH' \bot BC$, then $ OH' \bot \triangle A'BC$, so the optical center  $O$  of $(O,A'BC)$  locate on the vertical plane perpendicular to the plane $A'BC$ and going through the altitude from $A'$ , that is, the optical center $O'$  of $(O',ABC)$   being equivalent to $(O,A'BC)$   locate on the vertical plane $\pi_{BC}$ under three control points unchanged expression.

Similarly, we can prove that a side-sharing pair with a shared side $AC$ or $AB$  of the P3P problem under the optical center unchanged expression, then both the corresponding optical centers locate on the vertical plane $\pi_2$ and $\pi_3$ respectively under three control points unchanged expression.
\end{proof}

\subsection{A point-sharing pair and the geometric meaning}
In this section, we show that three theorems about the point-sharing pair:  Theorem \ref{thmpointtwo} is the necessary and sufficient condition,  Theorem \ref{thmpointone} is a sufficient condition,  Theorem \ref{thmpointgeometric} is the geometric meaning that the optical centers of point-sharing solutions must lie on one of three surfaces, which look like a modulated danger cylinder, we called it the skewed danger cylinder in this work.

\begin{thm}\label{thmpointtwo}
 Under the optical center unchanged expression, the necessary and sufficient condition for the P3P problem to have a point-sharing pair with a shared point $A$  is both the corresponding different  $(u_1, v_1)$ and $(u_2, v_2)$  satisfy the equation
   \begin{gather*}\label{equpoint}
   \frac{ cos\angle ACB}{cos\gamma}bu+\frac{cos \angle ABC}{cos\beta}cv-a=0
  \end{gather*}

\end{thm}
\begin{proof}
{ \bf (Necessity) }\\
Suppose the distances between the optical center and each control point of the point-sharing pair with a shared point $A$  are $(s_{11},s_{12},s_{13})$  and $(s_{11},s_{22},s_{23})$  respectively, then

 \begin{numcases}{}
    s_{11}^2 +s_{12}^2 -2cos\gamma s_{11}s_{12} =c^2 \label{a1}\\
    s_{11}^2 +s_{13}^2 -2cos\beta s_{11}s_{13}  =b^2 \label{a2} \\
    s_{12}^2 +s_{13}^2 -2cos\alpha s_{12}s_{13} =a^2 \label{a3}
  \end{numcases}

%\begin{gather}\label{equthmpoint2}
  \begin{numcases}{}
    s_{11}^2 +s_{22}^2 -2cos\gamma s_{11}s_{22} =c^2  \label{b1}\\
    s_{11}^2 +s_{23}^2 -2cos\beta s_{11}s_{23}  =b^2  \label{b2}\\
    s_{22}^2 +s_{23}^2 -2cos\alpha s_{22}s_{23} =a^2  \label{b3}
  \end{numcases}
%\end{gather}

By subtracting equations (\ref{b1}), (\ref{b2}) ,(\ref{a1}) from (\ref{a1}) ,(\ref{a2}), (\ref{a2}) respectively , we get

\begin{gather}
  %\begin{cases}
    s_{22}=2cos\gamma s_{11}-s_{12}  \label{c1} \\
    s_{23}=2cos\beta s_{11}-s_{13}  \label{c2} \\
    s_{13}^2-s_{12}^2 -2cos\beta s_{11}s_{13}+2cos\gamma s_{11}s_{12} =b^2 - c^2 \label{c3}
  %\end{cases}
\end{gather}
Substitute (\ref{a3}), (\ref{c1}), (\ref{c2}), (\ref{c3}) into the left hand side of (\ref{b3}), we get
\begin{gather}
\begin{split}
&s_{22}^2 +s_{23}^2 -2cos\alpha s_{22}s_{23}-a^2\\
&=\frac{-4acos\beta cos\gamma s_{11}^2}{s_{12}s_{13}} \left ( \frac{ cos\angle ACB}{cos\gamma}bu_1+\frac{cos \angle ABC}{cos\beta}cv_1-a \right ) \label{point}
\end{split}
\end{gather}

So $\frac{ cos\angle ACB}{cos\gamma}bu_1+\frac{cos \angle ABC}{cos\beta}cv_1-a=0$.
Similarly, we can prove $\frac{ cos\angle ACB}{cos\gamma}bu_2+\frac{cos \angle ABC}{cos\beta}cv_2-a=0$.

{\bf (Sufficiency)}
Since $\frac{ cos\angle ACB}{cos\gamma}bu_i+\frac{cos \angle ABC}{cos\beta}cv_i-a=0, i=1,2$, then substitute them into Equation (\ref{conic1}) , we have
\begin{gather}
 1+v_i^2-2cos\beta v_i=\frac{M-N}{1-N} ,i=1,2 \label{pointsi1}
\end{gather}
where
$N=(\frac{cos \angle ABC cos \gamma}{cos \angle ACB cos \beta})^2$,\\$ M=\frac{acos^2\gamma}{c^2cos\angle ACB cos^2\beta}\left (b-\frac{a}{cos \angle ACB }\right)+\frac{b^2}{c^2}$

Considering $s_{i1}^2=\frac{b^2}{1+v_i^2-2cos\beta v_i},i=1,2$, and Equation (\ref{pointsi1}) always holds for $v_1$  and $v_2$ , so we get $s_{11}=s_{21}$. By $(u_1,v_1)\neq(u_2,v_2)$  and Theorem \ref{thmsidetwo} and Lemma \ref{lemmaRieck}, we know $u_1\neq u_2$, $v_1\neq v_2$, so  $s_{12}\neq s_{22}$, $s_{13}\neq s_{23}$ , it follows that the two solutions corresponding to $(u_1,v_1)$  and   $(u_2,v_2)$ respectively are a point-sharing pair with a shared point $A$ .
\end{proof}

\begin{thm}\label{thmpointone}
Under the optical center unchanged expression, if $(u_1,v_1)$  satisfy $\frac{ cos\angle ACB}{cos\gamma}bu+\frac{cos \angle ABC}{cos\beta}cv-a=0$  which correspond to a solution  $(O,ABC)$ of the P3P Problem and $\beta< \angle ABC$, $\gamma< \angle ACB$,  then the P3P problem has another solution which has a shared point $A$  with $(O,ABC)$ , and the corresponding $(u_2,v_2)$  also satisfy $\frac{ cos\angle ACB}{cos\gamma}bu+\frac{cos \angle ABC}{cos\beta}cv-a=0$.
\end{thm}

\begin{proof}
 {\bf (By construction)}\\
Let $s_{11}=|OA|,s_{12}=|OB|,s_{13}=|OC|$  be the side lengths corresponding to the solution $(O,ABC)$.\\
Since $\frac{ cos\angle ACB}{cos\gamma}bu_1+\frac{cos \angle ABC}{cos\beta}cv_1-a=0$ , we have
$$\frac{ cos\angle ACB}{cos\gamma}bs_{12}+\frac{cos \angle ABC}{cos\beta}cs_{13}-as_{11}=0$$
Let $s_{22}=2cos\gamma s_{11}-s_{12}$  and $s_{23}=2cos\beta s_{11}-s_{13}$, then
\begin{gather*}
\begin{split}
&\frac{ cos\angle ACB}{cos\gamma}bs_{22}+\frac{cos \angle ABC}{cos\beta}cs_{23}-as_{11}=0 \label{pointone}
\end{split}
\end{gather*}
So $(u_2,v_2)$  corresponding to $(s_{11},s_{22},s_{23})$ also satisfy $\frac{ cos\angle ACB}{cos\gamma}bu+\frac{cos \angle ABC}{cos\beta}cv-a=0$.\\
Let us next show that $(s_{11},s_{22},s_{23})$  is a solution of this P3P problem.\\
Firstly, substituting $(s_{11},s_{22},s_{23})$   into the left hand sides of the three basic constraints
of the P3P problem respectively, we get
\begin{gather*}
\begin{split}
 &s_{11}^2 +s_{22}^2 -2cos\gamma s_{11}s_{22} -c^2 \\
&=s_{11}^2+(2cos\gamma s_{11}-s_{12})^2-2cos\gamma s_{11}(2cos\gamma s_{11}-s_{12})-c^2\\
&=0
\end{split}
\end{gather*}
Similarly,  $ s_{11}^2 +s_{23}^2 -2cos\beta s_{11}s_{23}  -b^2=0$  also holds.
For  $s_{22}^2 +s_{23}^2 -2cos\alpha s_{22}s_{23} -a^2$, by equation (\ref{point}) of Theorem \ref{thmpointtwo}, we know  $s_{22}^2 +s_{23}^2 -2cos\alpha s_{22}s_{23} -a^2=0$   holds when
$\frac{ cos\angle ACB}{cos\gamma}bu_1+\frac{cos \angle ABC}{cos\beta}cv_1-a=0$.\\
Secondly, if $s_{22}>0, s_{23}>0$, then $(s_{11},s_{22},s_{23})$  is another solution of the P3P problem which has a shared point $A$  with $(s_{11},s_{12},s_{13})$ .\\
It follows that we will prove when  $\beta< \angle ABC$ and $\gamma< \angle ACB$,  $s_{22}>0$ and $s_{23}>0$  hold.\\
Substituting $s_{23}=2cos\beta s_{11}-s_{13}$  into  $ s_{11}^2 +s_{23}^2 -2cos\beta s_{11}s_{23}  -b^2=0$, we get
$$ s_{11}^2- s_{13}s_{23}=b^2$$
$$ s_{23}=\frac{s_{11}^2-b^2} {s_{13}}$$
so $s_{23}>0$   is equivalent to $s_{11}>b$ .\\
Similarly, $s_{22}>0$   is equivalent to $s_{11}>c$ .\\
Next we prove that $s_{11}>b$  is equivalent to $cos\gamma > cos \angle BCA$ .
\begin{figure}[t]
\begin{center}
\includegraphics[width=0.5\linewidth]{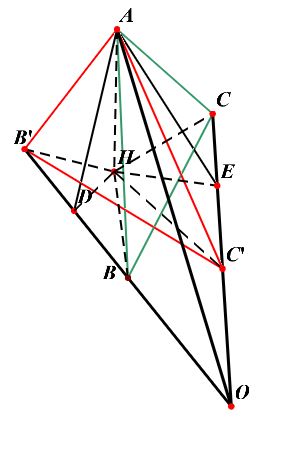}
\end{center}
   \caption{$\{O, (A, B, C)\}$ and $\{O, (A, B', C')\}$ are a point-sharing pair of solutions.}
\label{figpoint3}
\end{figure}
If  $(s_{11},s_{22},s_{23})$   is another solution of the P3P problem, then as shown in Figure \ref{figpoint3}, Let $AH \perp plane BOC$  ,  $H$ is the foot. And let $AD \perp OB, AE \perp OC$ , $D,E$  are the feet respectively. By $\triangle ABC$  is congruent to $\triangle AB'C'$  , we have $|BC|=|B'C'|$, and  $D,E$  are midpoint of segment $BB'$  and $CC'$  respectively, so $|HB|=|HB'|$  and $|HC|=|HC'|$, Furthermore,  $\triangle HBC$  is congruent to $\triangle HB'C'$, it is obvious $\angle BHC = \angle B'HC'$ , thus $\angle BHB' = \angle CHC'$  , then isosceles triangles $\triangle BHB'$  and $\triangle CHC'$  are similar ones. we get  $\angle B'BH = \angle C'CH$ , so
$\angle OBH + \angle OCH =\pi$, then four points $B,O,C,H$  are cocyclic, we have $\angle BOH = \angle BCH$.

we note that
 $$cos\gamma=\sqrt{1-\frac{|AH|^2}{s_{11}^2}}cos\angle BOH$$
and
 $$cos\angle BCA=\sqrt{1-\frac{|AH|^2}{b^2}}cos\angle BCH$$
so $s_{11}>b$  is equivalent to $cos\gamma > cos \angle BCA$. Similarly, we can prove that  $s_{11}>c$  is equivalent to $cos\beta > cos \angle ABC$. Thus when $cos\beta > cos \angle ABC$ and $cos\gamma > cos \angle BCA$, $(s_{11},s_{22},s_{23})$ is another solution which has a shared point  $A$  with $(s_{11},s_{12},s_{13})$ .
\end{proof}
\begin{remark}
Inequalities $\beta< \angle ABC$ and $\gamma< \angle ACB$ mean that the optical center  $O$ lies outside two toroids by rotating the arc $ABC$  and $ACB$ of circumcircle of the triangle $\triangle ABC$ around $AC$  and $AB$ respectively.
\end{remark}
\begin{thm}\label{thmpointgeometric}
{\bf {(The geometric meaning of the point-sharing pair)}}\\
Under the optical center unchanged expression, if $(O,ABC)$  and $(O,AB'C')$  are a point-sharing pair with a shared side $A$  of the P3P Problem, then both the optical centers $O$  and $O'$  of  $(O,ABC)$  and $(O',ABC)$   locate on a cubic surface, where $(O',ABC)$ is equivalent to $(O,AB'C')$.
\end{thm}
\begin{figure}[t]
\begin{center}
\includegraphics[width=0.7\linewidth]{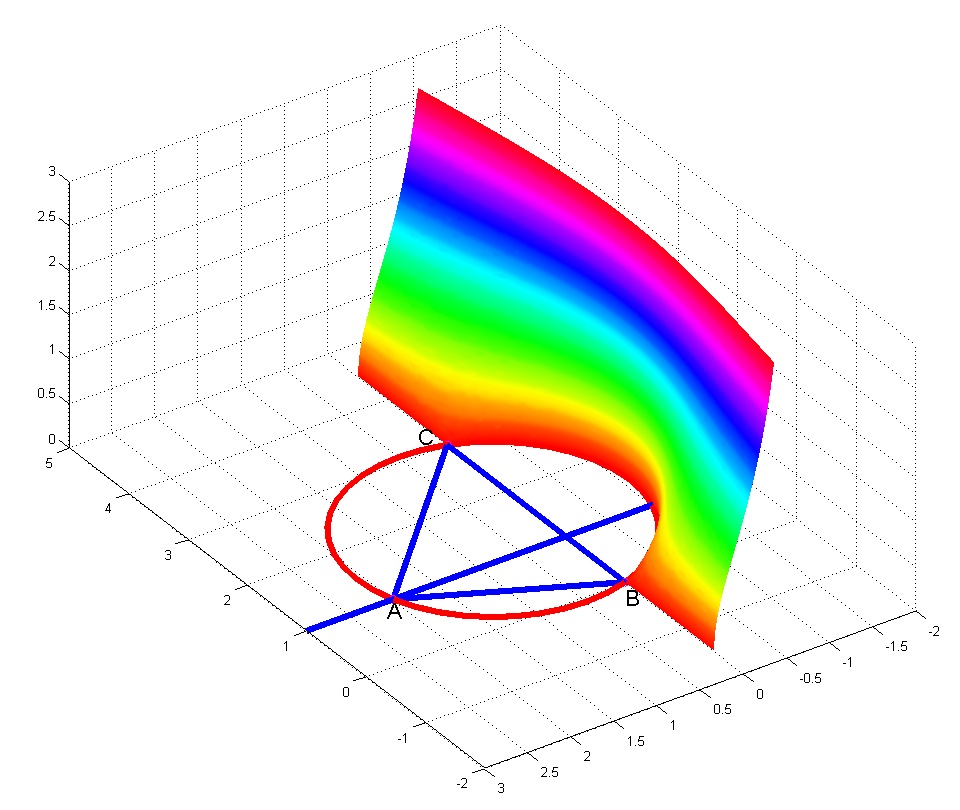}
\end{center}
   \caption{A graphical representation of the skewed danger cylinder in equation(\ref{skew}).}
\label{fig:skew}
\end{figure}
\begin{proof}
 Without loss of generality,suppose $(O,ABC)$  is one of the point-sharing pair,  let the coordinates of three control points are $A=(e,f,0),B=(0,0,0),C=(a,0,0)$,the optical center is $O=(x,y,z)$ , it is easy to get\\
$b^2=f^2+(a-e)^2, c^2=e^2+f^2$ \\
$s_{11}=\sqrt{(x-e)^2)+(y-f)^2+z^2}$\\
$s_{12}=\sqrt{x^2+y^2+z^2}$\\
$s_{13}=\sqrt{(x-a)^2+y^2+z^2}$\\
$cos\beta=\frac{s_{11}^2+s_{13}^2-b^2}{2s_{11}s_{13}}, cos\gamma=\frac{s_{11}^2+s_{12}^2-c^2}{2s_{11}s_{12}}$.
By Theorem \ref{thmpointtwo},  $\frac{ cos\angle ACB}{cos\gamma}bu_1+\frac{cos \angle ABC}{cos\beta}cv_1-a=0$  hold, that is,
$ \frac{ cos\angle ACB}{cos\gamma}bs_{12}+\frac{cos \angle ABC}{cos\beta}cs_{13}-as_{11}=0 $
Substituting all parameters  into the equation,we have
 {\footnotesize
\begin{align}
&y\cdot \underbrace{\left[\left(x-\frac{a}{2}\right )^2+\left(y-\frac{e^2-ae+f^2}{2f}\right)^2 -\frac{(e^2+f^2)((a-e)^2+f^2)}{4f^2}\right]} \notag \\
& \hskip 14.5em danger ~ cylinder  \notag \\
& =z^2(e^2 -fy-ae) \label{skew}
\end{align}}
Where equation $\left(x-\frac{a}{2}\right )^2+\left(y-\frac{e^2-ae+f^2}{2f}\right)^2 -\frac{(e^2+f^2)((a-e)^2+f^2)}{4f^2}=0$  is the danger cylinder of the P3P problem, so equation \ref{skew} is called the skewed danger cylinder(see Figure ), and it is a cubic surface, so the theorem is proved.
\end{proof}
\begin{remark}
In this work, Equations $cos\gamma u-cos\beta v=0$ and $\frac{ cos\angle ACB}{cos\gamma}bu+\frac{cos \angle ABC}{cos\beta}cv-a=0$  are called constraints of the side-sharing pair and the point-sharing pair.
\end{remark}
\subsection {Companion pairs of the side-sharing pair and the point-sharing pair}
The side-sharing pair and point-sharing pair of solutions have a close relationship, which can be formally stated as:
\begin{thm}
 If the P3P problem has a side-sharing pair (or a point-sharing pair) of solutions, their remaining solutions, if any, must be a point-sharing pair (or a side-sharing pair).
\end{thm}
\begin{proof}
  Let us consider the pair of conics (2) in section 2.1.
By subtracting (C1) from (C2),  we have:
\begin{align}\label{equ10}
 &(a^2+b^2-c^2)u^2+2(c^2-b^2)cos\alpha uv-(a^2-b^2+c^2)v^2 \notag \\
 &-2a^2cos\gamma u+2a^2cos\beta v=0
\end{align}
(I). Given a pair of side-sharing solutions,  if the P3P problem still has other solutions,  these remaining solutions must be a pair of point-sharing solutions.\\
According to Theorem \ref{thmsidetwo},  if there is a pair of side-sharing solutions with $BC$ being the shared side,  we have $cos\gamma u-cos\beta v=0$, then
\begin{align}\label{equside1}
 v= u cos \gamma/cos\beta
\end{align}
Substituting (\ref{equside1})  into (\ref{equ10}),  we have
\begin{align*}
\textstyle \frac{2(b^2-c^2)cos\alpha cos\beta cos\gamma-cos^2\gamma(b^2-a^2-c^2)+cos^2 \beta(c^2-a^2-b^2)}{cos^2\beta}u^2=0
\end{align*}
then
\begin{align}\label{equ11}
\textstyle cos\alpha=\frac{cos^2\gamma(b^2-a^2-c^2)-cos^2 \beta(c^2-a^2-b^2)}{2(b^2-c^2)cos\beta cos\gamma}
\end{align}
Substituting (\ref{equ11}) into (\ref{equ10}),  (\ref{equ10}) can be factorized as:
\begin{align}
  (ucos\gamma-vcos\beta)\left(\frac{ cos\angle ACB}{cos\gamma}bu+\frac{cos \angle ABC}{cos\beta}cv-a\right)=0 \label{factor}
\end{align}
So $\frac{ cos \tiny \angle ACB}{cos\gamma}bu+\frac{cos \tiny \angle ABC}{cos\beta}cv-a=0$ also holds,  then by Theorem \ref{thmpointtwo}, and if $u>0,v>0$, there must exist a pair of point-sharing solutions with the shared point on the ray $OA$.\\
(II).Given a pair of point-sharing solutions,  if the P3P problem still has other solutions,  these remaining solutions must be a pair of side-sharing solutions.\\
From Theorem \ref{thmpointtwo},  suppose there is a pair of point-sharing solutions with shared point $A$,  then $\frac{ cos \tiny \angle ACB}{cos\gamma}bu+\frac{cos \tiny \angle ABC}{cos\beta}cv-a=0$,
so substituting $v$ into  (\ref{equ10}),  we have:
\begin{align}\label{equ12}
 \scriptstyle 2(b^2-c^2)cos\alpha cos\beta cos\gamma-cos^2\gamma(b^2-a^2-c^2)+cos^2 \beta(c^2-a^2-b^2)=0
\end{align}
(\ref{equ12}) is the same with (\ref{equ11}),  hence (\ref{equ10}) can be factorized as (\ref{factor})again.

So $ucos\gamma-vcos\beta=0$ additionally holds,  then by Theorem \ref{thmsidetwo}, and if $u>0,v>0$,  there must exist a pair of side-sharing solutions with $BC$ being the shared side.

Combining (I) and (II),  the theorem is proved.
\end{proof}
%------------------------------------------------------------------------
%-------------------------------------------------------------------------

\section{Conclusion}

In this work,  we investigate the multi-solution phenomenon in the P3P problem,  and provide some new insights into the nature of some interesting phenomena. In particular we show that the necessary and sufficient condition for the P3P problem to have a pair of side-sharing solutions is that the two optical centers of the solutions lie on the 3 vertical planes to the base plane of control points; the necessary and sufficient condition for the P3P problem to have a pair of point-sharing solutions is that the two optical centers of the solutions lie on the 3 so-called skewed danger cylinders; And if the P3P problem has other solutions in addition to a pair of side-sharing (point-sharing) solutions,  these remaining solutions must be a point-sharing (side-sharing) pair,  or,  the side-sharing pair and the point-sharing pair are often companion pairs. These results are helpful to understand the multiple solution phenomenon in the P3P problem.

%------------------------------------------------------------------------
{\small
\bibliographystyle{ieee}
\bibliography{egbibmy}
}

\end{document}